\def\year{2019}\relax
\documentclass[letterpaper]{article} 
\usepackage{aaai19}  
\usepackage{times}  
\usepackage{helvet}  
\usepackage{courier}  
\usepackage{url}  
\usepackage{graphicx}  
\usepackage{amsmath}
\usepackage{amsfonts}
\usepackage[utf8]{inputenc}
\usepackage[english]{babel}
\usepackage{algorithm}
\usepackage[noend]{algpseudocode}
\usepackage{amsthm}
\usepackage{graphics}
\usepackage{caption}
\usepackage{float}
\usepackage{subcaption}
\usepackage{tabu}
\usepackage{adjustbox}
\usepackage{multirow}
\usepackage{balance}
\usepackage{stfloats}
\theoremstyle{plain}
\newtheorem{definition}{Definition}[section]
\newtheorem{proposition}{Proposition}[section]
\newtheorem{assumption}{Assumption}[section]

\begin{document}
%
\title{ANS: Adaptive Network Scaling for Deep Rectifier Reinforcement Learning Models}
\author{
Yueh-Hua Wu$^{1,2}$, Fan-Yun Sun$^{1}$, Yen-Yu Chang$^{1}$, Shou-De Lin$^{1}$\\
$^{1}$National Taiwan University, $^{2}$RIKEN-AIP\\
\{d06922005,\ sdlin\}@csie.ntu.edu.tw, \{b04902045,\ b03901138\}@ntu.edu.tw
}
\maketitle
\begin{abstract}
This work provides a thorough study on how reward scaling can affect performance of deep reinforcement learning agents. In particular, we would like to answer the question that \emph{How does reward scaling affect non-saturating ReLU networks in RL?} This question matters because ReLU is one of the most effective activation functions for deep learning models. We also propose an Adaptive Network Scaling framework to find a suitable scale of the rewards during learning for better performance. We conducted empirical studies to justify the solution. 
\end{abstract}

\section{Introduction}
Deep reinforcement learning (RL) has achieved tremendous success in a variety of domains such as recommendation systems \cite{chen2018stabilizing,dulac2015deep} and board games \cite{silver2017mastering}. However, training a robust DRL agent is generally considered as non-trivial. Several factors such as network architecture (activation functions, number of layers, learning rate), codebases  \cite{duan2016benchmarking},
random seed, and reward scale \cite{henderson2017deep} are shown to have significant impact on the overall performance of RL agents. 
 This paper focuses on the \emph{reward scaling} problem as reported in \cite{henderson2017deep,gu2016q,duan2016benchmarking}, in which it is discovered that different reward scales can result in performance change by a large margin, transforming an unlearnable task into a learnable task. Given activation functions like \emph{sigmoid} or \emph{tanh}, the problem of saturation \cite{glorot2010understanding,vincent2015efficient,montavon1998tricks} seems to be the reason why large reward scales cause performance drop since it is harder for bounded output to handle targets of large scale. One plausible remedy is to choose the Rectified Linear Unit (ReLU) \cite{nair2010rectified} since it is not a saturating activation; that is,
\begin{align*}
    \lim_{x\rightarrow \infty}\text{ReLU}(x)=+\infty.
\end{align*}
A question therefore arises, \emph{How does ReLU react to different scales of reward?} The answer does matter because ReLU is faster and the consequent sparse representations have shown remarkable performance in training deep neural networks \cite{glorot2011deep}. We would also like to study how the reward scaling affects the \emph{dying ReLU} problem, and how some potential solutions such as Leaky-ReLU \cite{maas2013rectifier} and ELU \cite{clevert2015fast} can be affected by the scaling of rewards. One of the major conclusion in this paper is that different from the intuition that smaller reward scale seems to be better (which is usually true for \emph{sigmoid} or \emph{tanh}), for ReLU a larger scale on rewards generally leads to a better performance. 
In a nutshell, we would like to answer the following questions in this work.
\begin{itemize}
    \item[Q1.] What is the difference between adjusting learning rate and reward scaling?
    \item[Q2.] How is ReLU networks affected by reward scales and whether it is related to the level of dying ReLU?
    \item[Q3.] Does non-saturating Leaky-ReLU and ELU improve the performance of RL by avoiding the dying-ReLU problem and how do they react when different reward scales are applied?
    \item[Q4.] Since reward scale matters, how to find appropriate reward scales yielding better ReLU-based DRL models?
\end{itemize}

Section 2 gives answers to Q1, Q2, and Q3. We have shown that the change of learning rate actually shows opposite effect with the change of reward scale in terms of dying ReLU. We further show that the influence of reward scaling is significant -- unlearnable tasks can become learnable when appropriate scale is adapted. Finally, the performance of ELU and leaky-ReLU can be greatly changed with different reward scales as well. Based on the empirical evidence obtained in our experiments, we solve Q4 by proposing the \emph{adaptive network scaling} (ANS) framework, which enables ReLU networks to find the proper scale for reward targets and quickly update the parameters accordingly, without the need to retrain RL models after scaling. 


\subsection{Preliminaries}
RL defines a class of algorithms solving problems modeled as a Markov Decision Process (MDP). An MDP consists of a tuple $(\mathcal{S}, \mathcal{A}, \mathcal{T}, \mathcal{R}, \gamma)$, where $\mathcal{S}$ is the space of state $s$, $\mathcal{A}$ is the space of action $a$, $\mathcal{T}$ is the transition function with probability distribution $\Pr(s'\vert s, a)$, $\mathcal{R}(s,a)$ is the reward function that outputs a scalar feedback given action $a$ made at state $s$, and $\gamma$ is the discount factor.

In actor-critic framework, the actor is the policy function $\pi_\theta(a\vert s)$, which is used to react to environments and collect samples of experience $(s, a, r)$. Those samples are used to update the critic, the value function $Q_\pi(s,a)$ that outputs the estimate of cumulative reward of current policy $\pi_\theta$,
\begin{align}
    Q_\pi(s_t,a_t) = E_{s_{t+1}, a_{t+1},...}\left[\sum_{i=0}^\infty \gamma^i r(s_{t+i})\right].    
\end{align}
The output of value function can be further used to update policy function, which aims to maximize the expected return $J(\theta)$. Policy gradient method \cite{sutton2000policy,kakade2002natural} is usually preferable to update the policy in actor-critic methods. The policy gradient can be written as
\begin{align}
    \nabla_\theta J(\theta)=E_{\tau\sim\pi_\theta(\tau)}\left[\nabla_\theta \log \pi_\theta(\tau)r(\tau)\right] \label{policygrad},
\end{align}
where $\pi_\theta$ is the policy function and $\tau$ is the trajectories generated from the interaction between $\pi$ and the environment. Here the cumulative reward $r(s,a)$ is sometimes replaced with Q-value function $Q(s,a)$ or advantage function,
\begin{align}
    A(s,a)=Q(s,a)-V(s)
\end{align}
so as to reduce the variance of policy gradient.

Neural networks \cite{rosenblatt} is a powerful function approximator trained with back-propagation. Different optimizers such as Stochastic Gradient Descent (SGD) \cite{rumelhart1985learning} or Adam \cite{kingma2014adam} are adapted to stabilize the parameter update, which can further improve the outcome of the learning process.
SGD updates the parameters with shuffled order of instances and usually incorporates Nesterov momentum \cite{nesterov1983method} to prevent from taking a big jump in the direction of the updated gradient. On the other hand, Adam optimizer considers the first and the second moment of the gradient sequence with bias correction for parameter update
\begin{align}
    \theta_t \leftarrow \theta_{t-1}-\alpha\frac{\Hat{m_t}}{\sqrt{\Hat{v_t}+\epsilon}}\label{eq:adam},
\end{align}
where $\Hat{m_t}$ and $\Hat{v_t}$ is the unbiased first and the second moment of gradients respectively. 


\section{Influence of Reward Scaling}
Reward scaling indicates the change in MDP from $(\mathcal{S}, \mathcal{A}, \mathcal{T}, \mathcal{R}, \gamma)$ to $(\mathcal{S}, \mathcal{A}, \mathcal{T}, c\mathcal{R}, \gamma)$, where $c\in \mathbb{R}^+$. When it comes to reward scaling, the first question is usually -- what is the difference between scaling learning rate and scaling reward function (Q1)? By considering the gradient update with loss function $\mathcal{L}(f(\theta), y)$, we have
\begin{align}
    \theta_t\leftarrow \theta_{t-1} - \alpha \nabla_{f}\mathcal{L}(f(\theta), y)\nabla_ \theta f(\theta).
\end{align}
If mean-square-error (MSE) loss is used, the update rule becomes
\begin{align}
    \theta_t\leftarrow \theta_{t-1}-\frac{2
    \alpha}{N}\sum_i(f_\theta(x_i)-y_i)\nabla_\theta f_\theta(x_i).
\end{align}
By scaling learning rate, the magnitude of the update term is scaled equally, and the signs of the gradients remain the same. However, if the target $y$ is scaled, the distance between the target and the estimate is changed instead of being scaled, which brings different sign distribution from the one without scaling. It is possible to scale the target $y$ so as to obtain a more informative gradient distribution. It can also potentially prevent the zig-zagging phenomena \cite{azzigzag} by making gradients have different signs (i.e. if the gradients are of the same sign, the parameters of neural networks are updated to the same direction, which may cause overshooting and require another update with opposite direction). Such zig-zagging results in optimization in an inefficient manner. On the other hand, if inappropriate reward scale is applied, the performance drops drastically as shown in \cite{henderson2017deep}.

The other aspect that reward scaling affects the learning performance are the \emph{saturation} problem \cite{glorot2010understanding,vincent2015efficient,montavon1998tricks} and the \emph{dying ReLU} problem. For activation functions like \emph{sigmoid} or \emph{tanh}, the output of the value is bounded and the gradients at these regions are almost zero. By chain rule, the gradient of the weight and bias in previous layers are diminished if the gradient of the activation is close to zero, which blocks the gradient signal from flowing to the subsequent layers. For ReLU, a large gradient may prevent neurons to output values larger than zero for all instances. Consequently, the gradient passed through the neurons becomes zero, resulting in the \emph{dying ReLU} problem. Even though dying ReLU is not completely irreversible \footnote{Technically, except for the first layer, the update of other parameters can change the distribution of the input of the dying neurons, which is possible to make them alive again.}, it may take large amount of time to recover and thus slow down the learning process.

Since the data are coming in an online fashion for RL, it is hard to apply the definition of dying-ReLU -- for all input $x$ of the batch data, the output of neuron $n$ is negative. Consequently, here we define \emph{pseudo-dying ReLU}, which is used to evaluate the tendency of neurons that may become dying.

\begin{definition}
For a set of collected samples $\mathcal{B}$ with $\lvert\mathcal{B}\rvert=B$, if
\begin{align}
    f_{n}(f_p(\pmb{x}))\leq 0, \forall \pmb{x}\in \mathcal{B},
\end{align}
where $f_p$ is the output of the previous layer and $f_{n}$ is the function of neuron $n$. For this case, we define neuron $n$ with ReLU as a pseudo-dying ReLU with respect to $\mathcal{B}$.
\end{definition}
\begin{proposition}\label{prop:proof}
Consider $f_n(\pmb{p})=\pmb{w}^\top \pmb{p}+b$ and $\lVert \pmb{p}\rVert=f_p(\mathcal{B})$ is normally distributed \footnote{A more proper setting is to assume a truncated normal distribution. However, to simplified the derivation we use normal distribution and assume that the distribution is far enough from zero.}. For the distribution of the angle between $\pmb{w}$ and $\pmb{p}$, we follow empirical distribution if necessary. If neuron $n$ is pseudo-dying with respect to $\mathcal{B}$, then the probability that $\mathcal{B}\cup \left\{\pmb{x}_{B+1}\right\}$ does not cause pseudo-dying ReLU is upper bounded by
\begin{enumerate}
    \item $\lvert \pmb{w}^\top\pmb{p}_i\rvert\geq\lvert b\rvert$ and $\pmb{w}^\top\pmb{p}_i<0$\\
    \begin{align}
        \Pr(\text{not pseudo-dying ReLU})\leq \frac{1}{2}\left[1+\text{erf}\left(\frac{-1}{\sqrt{2B}}\right)\right].
    \end{align}
    \item $\lvert \pmb{w}^\top\pmb{p}_i\rvert\leq\lvert b\rvert$ and $b<0$\\
    \begin{align}
        &\Pr(\text{not pseudo-dying ReLU})\nonumber\\
        &\leq \frac{1}{2}\left[1-\text{erf}\left(\sqrt{\frac{2(B-1)}{B}}\left(1-\frac{\Bar{\mu}}{\lvert b\rvert/\lVert \pmb{w}\rVert\lvert\cos\theta_\text{min}\rvert}\right)\right)\right].
    \end{align}
\end{enumerate}
For other cases, $\Pr(\text{not pseudo-dying ReLU})=0$.
\end{proposition}
\begin{proof}
If neuron $n$ is a pseudo-dying ReLU with respect to $\mathcal{B}$, we have
\begin{align}
    f_n(\pmb{p}_i)=f_n(f_p(\pmb{x}_i))<0, \forall \pmb{x}_i\in\mathcal{B}.\label{pdrr}
\end{align}
We consider two cases that satisfy Equation \ref{pdrr}
\begin{enumerate}

    \item $\lvert \pmb{w}^\top\pmb{p}_i\rvert\geq\lvert b\rvert$ and $\pmb{w}^\top\pmb{p}_i<0$\\
    With Cauchy-Schwarz inequality, we have $\lVert \pmb{p}_i\rVert \geq\lvert b\rvert/\lVert \pmb{w}\rVert$. Therefore, the probability of sampling an instance that is not dying ReLU is
    \begin{align}
        &\Pr(\text{not pseudo-dying ReLU})\nonumber\\=&\Pr(\lVert \pmb{p}_i\rVert<\frac{\lvert b\rvert}{\lVert \pmb{w}\rVert})\\
        =&\frac{1}{2}\left[1+\text{erf}\left(\frac{\lvert b\rvert/\lVert \pmb{w}\rVert-\Bar{\mu}}{\Bar{\sigma}\sqrt{2}}\right)\right],
    \end{align}
    where $\Bar{\sigma}$ and $\Bar{\mu}$ are the estimate standard deviation and the estimate mean respectively. The estimate variance is bounded by
    \begin{align}
        \Bar{\sigma}^2&=\frac{1}{(B-1)}\sum (\lVert \pmb{p}_i\rVert-\Bar{\mu})^2\\
        &\leq \frac{1}{(B-1)}\left[B(B-1)(\frac{\lvert b\rvert}{\lVert \pmb{w}\rVert}-\Bar{\mu})^2\right]\\
        &=B(\frac{\lvert b\rvert}{\lVert \pmb{w}\rVert}-\Bar{\mu})^2.
    \end{align}
    Therefore,
    \begin{align}
        \Pr(\text{not pseudo-dying ReLU})\leq \frac{1}{2}\left[1+\text{erf}\left(\frac{-1}{\sqrt{2B}}\right)\right].
    \end{align}
    
    \item $\lvert \pmb{w}^\top\pmb{p}_i\rvert\leq\lvert b\rvert$ and $b<0$\\
    In this case, it should be noted that if $\pmb{p}_i$ and $\pmb{w}$ are orthogonal, then the probability of being not pseudo-dying ReLU is zero. We further consider the case that $\lvert\cos\theta\rvert\geq \lvert\cos \theta_\text{min}\rvert$. Similarly, we have $\lVert \pmb{p}_i\rVert\leq \lvert b\rvert/\lVert \pmb{w}\rVert\lvert\cos\theta_\text{min}\rvert$.
    \begin{align}
        &\Pr(\text{not pseudo-dying ReLU})\nonumber\\=&\Pr(\lVert \pmb{p}_i\rVert> \lvert b\rvert/\lVert \pmb{w}\rVert\lvert\cos\theta_\text{min}\rvert)\\
        =&\frac{1}{2}\left[1-\text{erf}\left(\frac{\lvert b\rvert/\lVert \pmb{w}\rVert\lvert\cos\theta_\text{min}\rvert-\Bar{\mu}}{\Bar{\sigma}\sqrt{2}}\right)\right].
    \end{align}
    Since the samples $\lVert \pmb{p}_i\rVert$ are bounded by zero and $\lvert b\rvert/\lVert \pmb{w}\rVert\lvert\cos\theta_\text{min}\rvert$, the variance is bounded by
    \begin{align}
        \Bar{\sigma}^2\leq \frac{B}{(B-1)}\frac{(\lvert b\rvert/\lVert \pmb{w}\rVert\lvert\cos\theta_\text{min}\rvert)^2}{4}.
    \end{align}
    Therefore,
    \begin{align}
        &\Pr(\text{not pseudo-dying ReLU})\nonumber\\
        \leq& \frac{1}{2}\left[1-\text{erf}\left(\sqrt{\frac{2(B-1)}{B}}\left(1-\frac{\Bar{\mu}}{\lvert b\rvert/\lVert \pmb{w}\rVert\lvert\cos\theta_\text{min}\rvert}\right)\right)\right].
    \end{align}
\end{enumerate}
\end{proof}
\noindent The physical meaning of the proposition is that if neuron $n$ is currently pseudo-dying, then the probability that it is still pseudo-dying with subsequent samples is high.

In the following context, we use pseudo-dying ReLU as an indicator of the true dying ReLU. Using this indicator is computationally cheaper and more efficient than checking whether the neuron is really dying or not. Here we also define pseudo-dying ReLU ratio (PDRR). For layer $i$, we have,
\begin{align}
    \text{PDRR}_i=\frac{\text{number of pseudo-dying ReLU in layer }i}{\text{number of neurons in layer }i}.
\end{align}


It should be noted that, in actor-critic approaches, the value network is where we can differentiate between changing learning rate and reward scaling. If the target of value network is scaled, the subsequent gradients can have different sign distribution as mentioned instead of simply being scaled. For policy network, as shown in Equation \ref{policygrad}, changing the scale of the estimate is the same as changing the magnitude of policy gradients. Especially, for RMSprop \cite{tieleman2012lecture} and Adam \cite{kingma2014adam} optimizers, as suggested in Equation \ref{eq:adam}, such scaling actually makes no difference because the division of the first moment to the squared root of the second moment cancels out the scale. In our experiments, to reduce variables that may affect experiment results, if the reward is scaled $c$ for the critic, we multiply the estimate from the critic with $1/c$ so that the magnitude of the estimate to the actor is the same.

\subsection{ReLU Experiments}
In this section, we aim to answer the following questions by conducting extensive experiments: (Q1) What is the difference between changing learning rate and changing reward scale in terms of cumulative rewards and PDRR? (Q2) How does ReLU network respond to different reward scales in terms of cumulative rewards and PDRR? (Q3) How is the performance of the variants of ReLU without dying-ReLU issue such as Leaky-ReLU and ELU? Are such activations insensitive to reward scales? To answer these questions, we compare different settings on DDPG \cite{lillicrap2015continuous} and A2C \cite{mnih2016asynchronous} to make sure the results are consistent with different actor-critic models. In this work, our experiments are built upon Mujoco environment \cite{todorov2012mujoco}, which is a physics engine for continuous control and provides multi-joint dynamics simulation.

\subsubsection{Background}
DDPG extends the deterministic policy gradient to continuous action spaces with actor-critic approach and utilizes Ornstein-Uhlenbeck process for exploration. The actor is updated by the differential of the expected return from the start distribution $J$ with respect to the actor parameters:
\begin{align}
    \nabla_\theta J\approx E_{s_t\sim\rho^\beta}\left[\nabla_\theta Q(s,a)\vert _{s=s_t,a=\mu(s_t)}\right]\nonumber,
\end{align}
where $\rho^\beta$ is the state distribution following stochastic policy $\beta$. The critic is updated with expected return that is estimated from target networks:
\begin{align}
    y_t=r_t+\gamma Q'(s_{t+1}, \mu'(s_{t+1}))\nonumber,
\end{align}
where $Q'$ is the target critic and $\mu'$ is the target actor. To make the learning more stable, ``soft'' update is used to update target networks so that the output of the target networks will not change too fast.

The other model that we use here is advantage actor critic (A2C), which is a synchronous and batched version of asynchronous actor critic (A3C) \cite{mnih2016asynchronous}. It utilizes policy gradient as indicated in Equation \ref{policygrad} to update the policy. Advantage function $A(s,a)$ is used in place of reward $r(\tau)$ to scale the gradient $\nabla_\theta \log \pi_\theta$ because it obtains lower variance for gradient estimation. What makes A3C and A2C different from other actor-critic methods is that it accumulates gradients of workers and use them to update the global model.

Exponential Linear Unit (ELU) is similar to ReLU except for negative inputs,
\begin{align}
    \text{ELU}(x)=\max (0, x) + \min(0, \alpha(\exp (x)-1))\nonumber.
\end{align}
An interesting result of this form is that the gradient can be computed with addition:
\[
    \text{ELU}'(x)=
    \begin{cases}
        1,&\text{if }x>0\\
        \text{ELU}(x) + \alpha,&\text{otherwise.}
    \end{cases}
\]
On the other hand, Leaky-ReLU allows a mild slope when input is negative,
\begin{align}
    \text{LeakyReLU}(x)=\max (0,x)+\alpha\min(0,x)\nonumber,
\end{align}
where $\alpha$ is usually set to a small value such as $0.01$.

In the following experiments, we use three-layer neural networks as function approximators. Since the two ReLU layers have similar patterns, we show ReLU in the second layer, which shows significant difference across variables. For clarity and conciseness, the most representative figures are presented, and we leave others in appendix.

\subsubsection{Results on Adjusting Learning Rates}
Figures \ref{fig:lr_a2c} shows how cumulative reward and PDRR change with respect to different learning rate. The results are not surprising as too large learning rate results in high PDRR, which impedes further performance improvement because of limited available computation units. It should be noted that although high PDRR leads to bad performance, the same correlation does not apply when PDRR is lower than a threshold. For $10^{-3}$ and $7\cdot 10^{-4}$ learning rate, even though they have higher PDRR than $10^{-4}$ and $10^{-5}$, their returns are higher. It may be caused by the fact that it is possible for neural networks to have excessive computation capacity \cite{guss2018characterizing} for certain tasks; consequently, PDRR may pose no harm under such situation.

\begin{figure}[!tbh]
\centering
\begin{subfigure}{0.4\textwidth}
    \centering
    \includegraphics[width=1\linewidth]{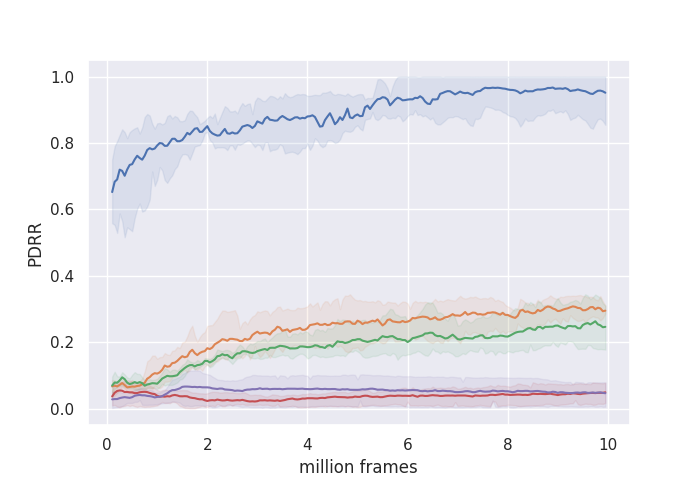}
    \caption{PDRR}
    \label{fig:lr_a2c_relu2}
\end{subfigure}
\begin{subfigure}{0.4\textwidth}
    \centering
    \vspace{-1mm}
    \includegraphics[width=1\linewidth]{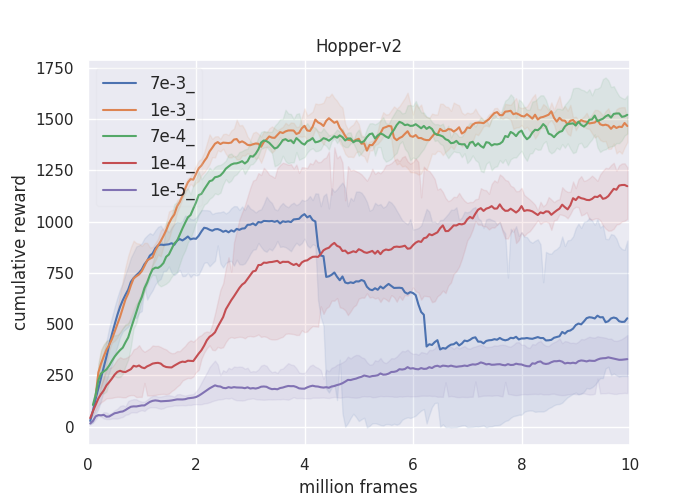}
    \caption{Return}
    \label{fig:lr_a2c_reward}
\end{subfigure}
\caption{A2C learning rate comparison on Hopper-v2.}
\label{fig:lr_a2c}
\end{figure}

\subsubsection{Reward Scale}
In Figure \ref{fig:a2c} and \ref{fig:ddpg}, they show that reward scales can have tremendous influence on the performance of RL models with ReLU networks. With proper reward scales, the results can be changed from learning nothing to a well-trained model. Another fact worth noting is that large reward scale does not cause significant PDRR compared with other scales; on the contrary, 0.5 reward scale shows a tendency to cause higher PDRR as suggested in Figure \ref{fig:rs_pdrr}. Such phenomenon also reflects on cumulative reward as A2C with 0.5 reward scale is unable to converge to a satisfactory solution. 

The phenomenon of small reward scale resulting in dying-ReLU can be explained as follows. Consider a layer with ReLU outputs a scalar as $\text{ReLU}(\pmb{w}^\top\pmb{x}+b)=y$. If the target $\Hat{y}$ is scaled by $c<1$ and $\pmb{x}$ remains unchanged, then the positive outputs of $\pmb{w}^\top\pmb{x}+b$ should be reduced while the negative part is not affected (the value does not even matter). In this case, $b$ is updated to make $y$ smaller, which potentially increase the probability of getting trapped by dying ReLU. In contrast, reducing learning rate decreases the magnitude of parameter updates in all directions, which proportionally make dying ReLU less probable. 

The results reveal that reward scales over certain unknown threshold usually allows RL models with ReLU to learn better. It can be caused by the fact that larger reward scales enlarge the difference between positive and negative returns and therefore the information of rewards becomes clearer. It can be confirmed by the outcome of the experiments that the influence of reward scaling on A2C is much larger than that on DDPG because A2C makes use of advantage function that has zero expected value with respect to actions. However, larger reward scale does not always mean better. Too large reward scale may introduce large gradients as well, which makes RL models unable to converge easily. As a result, to improve the performance of RL models, finding the reward scale that gives enough information of reward functions and meanwhile is not too large to converge to a good solution is where we can put effort.

\begin{figure*}[!ht]
\centering
\begin{subfigure}{0.34\textwidth}
    \centering
    \includegraphics[width=1\linewidth]{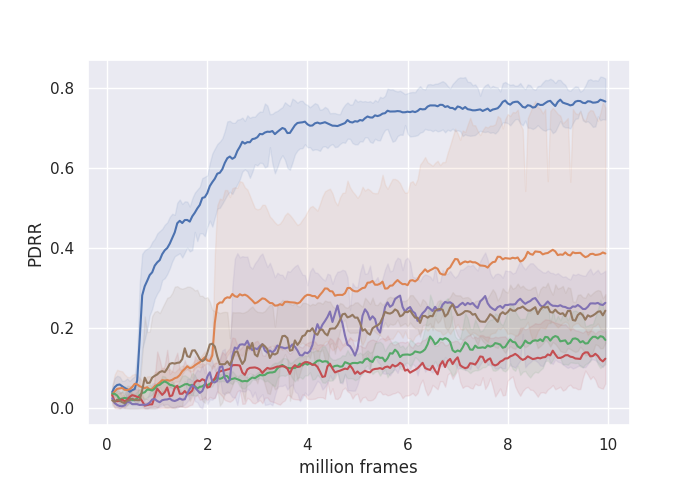}
    \caption{PDRR}
\end{subfigure}%
\begin{subfigure}{0.33\textwidth}
    \centering
    \includegraphics[width=1\linewidth]{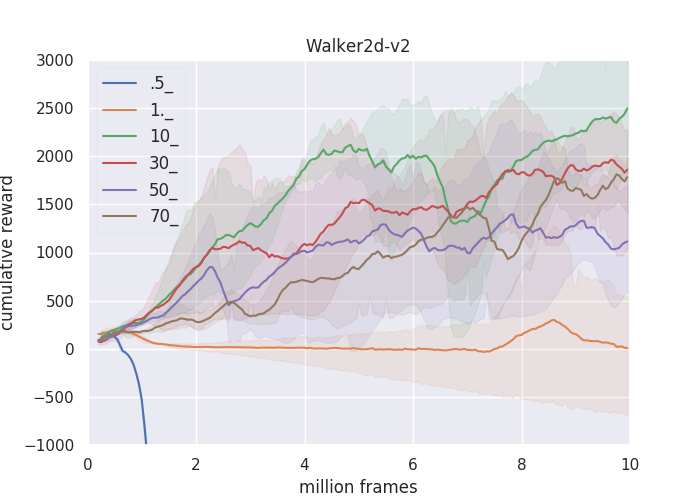}
    \caption{Return}
\end{subfigure}%
\begin{subfigure}{0.33\textwidth}
    \centering
    \includegraphics[width=1\linewidth]{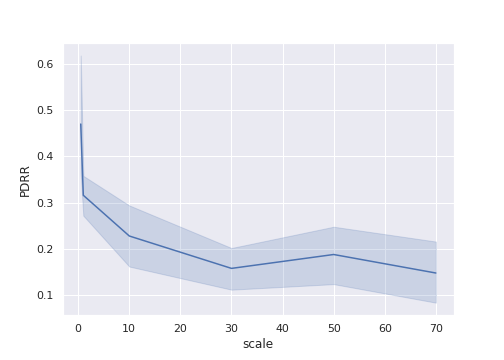}
    \caption{Reward scales vs. PDRR}
    \label{fig:rs_pdrr}
\end{subfigure}
\caption{A2C reward scales comparison on Walker2d-v2 and the correlation between reward scales and PDRR in five Mujoco environments.}
\label{fig:a2c}
\end{figure*}

\begin{figure}[!tbh]
    \centering
    \includegraphics[scale=0.4]{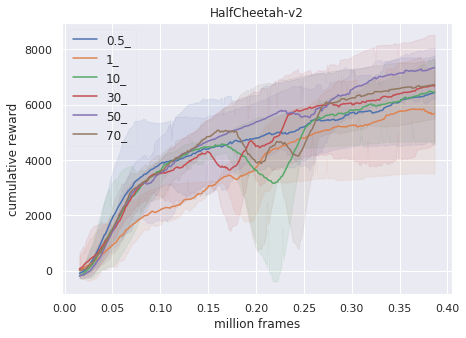}
    \caption{DDPG reward scales comparison on HalfCheetah-v2.}
    \label{fig:ddpg}
\end{figure}

\subsubsection{Leaky-ReLU and ELU}
We try variants of ReLU, Leaky-ReLU, and ELU to illustrate (1) how the performance of models with ELU or Leaky-ReLU changes with different reward scales, and (2) if the variants of ReLU without dying ReLU issue can outperform ReLU in RL problems. As Figure \ref{fig:act} suggests, ELU generally outperforms ReLU and Leaky-ReLU slightly; on the other hand, even though Leaky-ReLU does not suffer from dying ReLU, it still performs empirically worse than ReLU when it comes to cumulative rewards. Reward scaling is still an important factor to the performance of ELU and Leaky-ReLU. With proper reward scales, the RL models also converge to a much better solution.

\begin{figure}[!tbh]
    \centering
    \hspace{6mm}%
    \includegraphics[scale=0.4]{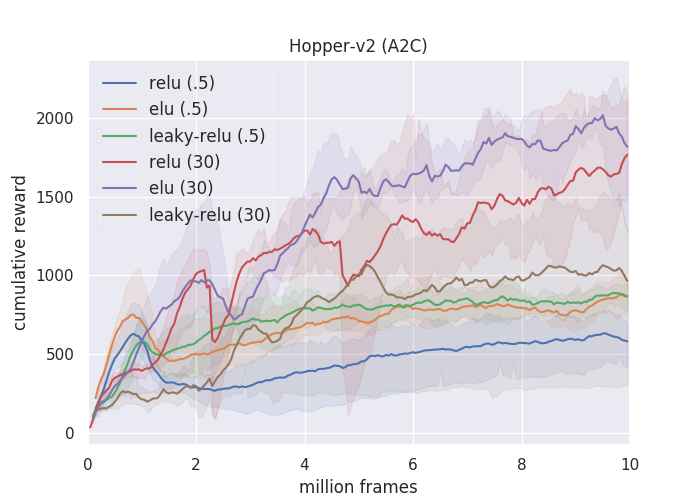}
    \caption{Activation comparison on Hopper-v2 with A2C.}
    \label{fig:act}
\end{figure}

\subsubsection{A Short Summary}
Learning rate and reward scale show utterly different pattern in terms of PDRR -- larger learning rate can result in higher PDRR whereas the higher reward scale have lower PDRR. The recommendation drawn from our experiments is that ELU slightly outperforms ReLU but is also computationally more costly than ReLU, so there is a trade-off. Even though Leaky-ReLU allows information to flow when the input is negative, empirically its performance is worse than ReLU. Since the magnitude of reward functions depends on the design of the environment, it is hard to give suggestions on reward scales. In spite of that, it seems that for ELU, Leaky-ReLU, and ReLU, a slightly larger reward scale such as 5 or 10 usually poses no harm and potentially brings benefits such as better performance.

\section{Adaptive Network Scaling}
We have shown that reward scaling has the potential to improve the performance of RL models by a large margin. In this section, we propose the \emph{Adaptive Network Scaling} (ANS) framework to answer the question: (Q4) How  to efficiently find an appropriate reward scale to improve the performance? This question can be decomposed into two sub-problems: (1) How to find a suitable reward scale? and (2) How to scale the reward while learning and transfer the learned parameters instead of learning from scratch?
For the first sub-problem, we propose a search strategy that requires few queries to find a satisfactory solution. For the second sub-problem, our solution is to transfer the model trained previously to the current reward scale so that retraining is not required. ANS is the combination of these two solutions.





\subsection{Adaptive Search Strategy}
Our previous experiments suggest that to achieve the best performance, the reward scale should be able to balance the magnitude of gradients and the maximization of the diversity between positive and negative samples. We first assume that the landscape of reward scale search space is concave as indicated in Assumption \ref{assum:concave}. 

\begin{assumption}\label{assum:concave}
We assume that for some reward scale $s_x$ and $s_y$, and for any $\alpha\in[0,1]$,
\begin{align}
    g(\alpha s_x + (1-\alpha) s_y)>\alpha g(s_x)+(1-\alpha)g(s_y),
\end{align}
where $g$ is the function that indicates the highest cumulative reward that the RL model can achieve with this reward scale.
\end{assumption}

Even though it is a strong assumption for the landscape, we still need to be careful not to use too many queries while optimizing it since estimating $g(s)$ requires solving an MDP. To derive a near-optimal solution under Assumption \ref{assum:concave}, we propose an adaptive search strategy that can find a satisfactory scale efficiently and is easy to implement. 

At the beginning of the strategy, we choose to multiply the reward scale with $c_\text{inc}>1$. If the ``performance'' is better than the previous scale, we continue to multiply the reward scale with the same $c_\text{inc}$ until the improvement stops. ANS then considers the case as overshooting and multiply the reward scale with $c_\text{dec}$ slightly less than one to refine the result and so on. It should be noted that the near-optimal scale can be found even when the optimal scale is less than one.

To derive the performance of the current scale, we leverage exponential moving average with bias correction \cite{kingma2014adam},
\begin{align*}
    m_t&\leftarrow \beta m_{t-1} + (1-\beta)R_t\\
    \Hat{m_t}&\leftarrow m_t/(1-\beta^t),
\end{align*}
where $R_t$ is the return of the $t$-th episodes and $\Hat{m_t}$ is the unbiased estimate of the mean. We record the highest unbiased estimate of mean $\Hat{m}_\text{max}$. If it is not updated in $T$ episodes, then ANS considers this case as getting trapped by a local optimum and $\Hat{m}_\text{max}$ is treated as the performance of the current scale.

The main advantage of this search method is that we do not need many queries of $g$ to find a satisfactory solution. In addition, if the optimal reward scale is bounded by $s_\text{max}$ and one, we have a bounded number of steps $n_\text{max}$,
\begin{align}
    n_\text{max}\leq \lceil \log_{c_\text{inc}}s_\text{max}\rceil - \lfloor \log_{c_\text{dec}}c_\text{inc}\rfloor.
\end{align}
Empirically, ANS requires less than six steps to achieve competitive performance of the best reward scale in our ReLU experiments and the computation cost is much lower than grid searching for the best scale.

\subsection{Network Scaling}
When the reward function is scaled, the target of state-action value $Q(s,a)$ and the state value $V(s)$ are scaled equally because they are exponentially weighted sum of rewards. 
Since learning from scratch is time consuming, here we introduce \emph{network scaling}, which allows us to preserve scaled outputs of the function,
\begin{align}
    cf_{W',b'}(\pmb{x})=f_{W,b}(c\pmb{x})\label{eq:scale}.
\end{align}
To achieve Equation \ref{eq:scale}, we utilize the scaling property of ReLU,
\begin{align}
    c\cdot \text{ReLU}(\pmb{x})=\text{ReLU}(c\pmb{x}).
\end{align}
The ReLU network can be expressed as 
\begin{align}
    f_{W,b}(\pmb{x})=\pmb{W}_n(...(\pmb{W}_2(\pmb{W}_1 \pmb{x}'+\pmb{b}_1))+\pmb{b}_2)...)+\pmb{b}_n,
\end{align}
where we omit ReLU in each layer for clarity and conciseness. If we scale each parameter,
\begin{align}
    f_{W',b'}(\pmb{x})&=\pmb{W}'_n(...(\pmb{W}'_2(\pmb{W}'_1 \pmb{x}+\pmb{b}'_1)+\pmb{b}'_2)...)+\pmb{b}'_n\\
    &=c_n\pmb{W}_n(...(c_1\pmb{W}_1 \pmb{x}+r_1\pmb{b}_1)...)+r_n\pmb{b}_n
\end{align}
by obeying the following constraints, we are able to preserve the scaled output as Equation \ref{eq:scale},


\begin{align}
    \prod_{i=1}^tc_i=r_t, \text{for }t=1,...,n
\end{align}
\begin{align}
    b_n=c.
\end{align}
We consider $c_i=\sqrt[n]{c}$ in the following derivation and our experiments. However, it is straightforward to extend to different values of $c_i$, which can be determined by the condition of each layer.

The change of the weights and the corresponding gradients are,
\begin{align}
    \pmb{W}_i^\text{new}&\leftarrow \sqrt[n]{c}\pmb{W}_i\\
    \nabla_{\pmb{W}_i^\text{new}}\mathcal{L}&\leftarrow c^{\left(2-\frac{1}{n}\right)}\nabla_{\pmb{W}_i}\mathcal{L}. 
\end{align}
For each bias, we have similar result,
\begin{align}
    \pmb{b}_i^\text{new}&\leftarrow \sqrt[n]{c^i}\pmb{b}_i\\
    \nabla_{\pmb{b}_i^\text{new}}\mathcal{L}&\leftarrow c^{\left(2-\frac{i}{n}\right)}\nabla_{\pmb{b}_i}\mathcal{L}.
\end{align}

We may consider reward scaling as a transfer approach that provides an initialization closed to the new solution space, which saves a lot of time to train RL models from scratch. The result with and without network scaling is shown in Figure \ref{fig:ablation}. Note that unless care is taken, manipulating parameters of neural networks can result in instability in the first few epochs of training; consequently, we strictly clip the gradients by imposing maximal norm on the policy network after scaling and gradually relax the norm, which prevents the trained policy network from being misled by the scaled value network at the beginning.

\begin{figure}[!tbh]
    \centering
    \includegraphics[scale=0.4]{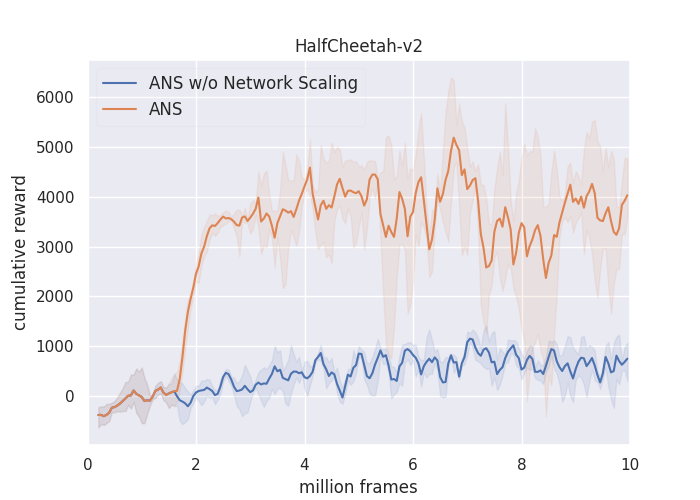}
    \caption{Comparison of ANS with and without network scaling on A2C.}
    \label{fig:ablation}
\end{figure}

\begin{algorithm}[!htb]
\caption{Adaptive Network Scaling}\label{algo:ans}
\begin{algorithmic}
\State Given tolerance $T$, $c_\text{inc}$ and $c_\text{dec}$
\State Initialize $t=0$, $t_\text{stop}=0$, $m_t=0$, $R_\text{prev}=-\infty$, $\Hat{m}_\text{max}=-\infty$, $c=c_\text{inc}$ and $\emph{reverse}=\text{False}$. 
\For{\text{each episode}}
\State $t\leftarrow t+1$
\State $t_\text{stop}\leftarrow t_\text{stop}+1$
\State $m_t\leftarrow \beta m_{t-1}+(1-\beta)R_t$
\State $\Hat{m_t} \leftarrow m_t/(1-\beta^t)$
\If{$\Hat{m_t} > \Hat{m}_{max}$}
    \State $\Hat{m}_{max}\leftarrow \Hat{m_t}$
    \State $t_\text{stop} = 0$
\EndIf
\State \textbf{end if}
\State
\If{$t_\text{stop}>T$}
    \If{\emph{reverse} and $\Hat{m}_\text{max}\leq R_\text{prev}$}
        \State \textbf{break}

    \ElsIf{not \emph{reverse} and $\Hat{m}_\text{max}\leq R_\text{prev}$}
        \State $c\leftarrow c_\text{dec}$
        \State \emph{reverse}$\leftarrow$True
    \EndIf
    \State scale the reward function with $c$
    \State \emph{NetworkScaling}($c$)
    \textbf{end if}
    \State $R_\text{prev}\leftarrow \Hat{m}_\text{max}$
    \State reinitialize $t, t_\text{stop}, m_t, \Hat{m}_{\text{max}}$
\EndIf
\textbf{end if}
\EndFor
\State \textbf{end for}
\end{algorithmic}
\end{algorithm}

\subsection{Related Work}
To the best of our knowledge, Pop-Art \cite{van2016learning} is the only framework that directly deals with reward scales. Their work is mainly for value-based RL approaches while we focus on actor-critic approaches, which is capable of achieving competitive performances \cite{duan2016benchmarking} such as DDPG \cite{lillicrap2015continuous}, PPO \cite{schulman2017proximal}, and A2C \cite{mnih2016asynchronous}. The main difference between Pop-Art and our approach is that ANS finds the best reward scale in terms of learning performance whereas Pop-Art maintains a normalized reward distribution, which may not guarantee to improve the outcome of learning.

Pop-Art considers the neural network as
\begin{align}
    f(x)=\pmb{\Sigma}g_{\theta,\pmb{W},\pmb{b}}(x)+\pmb{\mu}=\pmb{\Sigma}(\pmb{W}h_\theta(x)+\pmb{b})+\pmb{\mu},
\end{align}
where $g_{\theta,\pmb{W},\pmb{b}}$ is a normalized function and $h_\theta$ is a parameterized (non-linear) function. To preserve the output of the unnormalized function $f$ while maintaining the normalized function $g_{\theta,\pmb{W},\pmb{b}}$, that is,
\begin{align}
    f_{\theta,\pmb{\Sigma},\pmb{W},\pmb{b},\pmb{\mu}}(x)=f_{\theta, \pmb{\Sigma}_{\text{new}},\pmb{W}_{\text{new}}, \pmb{b}_{\text{new}}, \pmb{\mu}_{\text{new}}},
\end{align}
they apply incremental update to derive zero mean and unit variance with $\pmb{\Sigma}_{\text{new}}$ and $\pmb{b}_\text{new}$. Additionally, the output can be preserved by the following update
\begin{align}
    \pmb{W}_{\text{new}}&=\pmb{\Sigma}_{\text{new}}^{-1}\pmb{\Sigma}\pmb{W}\\
    \pmb{b}_{\text{new}}&=\pmb{\Sigma}_{\text{new}}^{-1}(\pmb{\Sigma}\pmb{b}+\pmb{\mu}-\pmb{\mu}_{\text{new}}).
\end{align}

They claim that their approach is able to deal with the problem that the magnitude of value target can change over time. However, even though maintaining the magnitude to a predetermined range is a straightforward approach to stabilizing the learning process, empirically, the performance is not generally good -- the performance drops in nearly fifty percent of the Atari games with Pop-Art approach as reported. This can be caused by the fact that it is hard to determine appropriate values that the mean and variance should be scaled to. Conceptually, we claim that the scaling method should consider the training status such as improvement in cumulative reward, which is the ultimate objective of RL.

\subsection{Experiments}
In our experiment, we use DDPG and A2C models to show how ANS and Pop-Art influence the performance. For performance evaluation, we follow \cite{henderson2017deep} to use final average across 5 trials of returns across the last 100 trajectories after $N$ frames. It should be noted that it is possible for ANS to terminate before $N$ frames and we continue to train the model until $N$ frames are reached.

Since A2C undergoes 16 environments at the same time, we use 100 updates as tolerance $T$. For DDPG, 100-episode tolerance is adapted for Walker2d-v2 and Hopper-v2, and 500-episode tolerance is used for HalfCheetah-v2. Similarly, the total number of frames $N$ for A2C is 10 million. For DDPG, we use one million frames for Walker2d-v2 and Hopper-v2, and two million frames for HalfCheetah-v2. Different number of frames and tolerance are adapted because DDPG keeps improving within one million frames on HalfCheetah-v2, which makes it difficult to apply ANS. We fix $\beta=0.9$, $c_\text{inc}=8.0$ and $c_\text{dec}=0.9$ to avoid cumbersome parameter tuning.  The performance of Swimmer-v2 and Ant-v2 in DDPG is not included since they can hardly be trained even with reward scaling. 

\begin{table*}[!tp]
\centering
\begin{adjustbox}{max height=1in}
\begin{tabular}{l|c|c|c|c|c}
\textbf{Algorithm} & \textbf{HalfCheetah-v2} & \textbf{Hopper-v2} & \textbf{Walker2d-v2} & \textbf{Swimmer-v2} & \textbf{Ant-v2}\\ \hline
DDPG (ReLU)          & 7944.54 & 264.19 & 705.96 & - & -\\ 
DDPG (ELU)          & 7899.17 & 795.83 & 1256.20 & - & -\\
DDPG (Sigmoid)      & 7730.29 & 1173.10 & 695.21 & - & -\\
DDPG (Tanh)         & 1802.56 & 808.24 & 723.77 & - & -\\
DDPG+Pop-Art    & 1273.33 & 287.92 & 356.08 & - & -\\ 
DDPG+ANS         & \textbf{9393.24} & \textbf{1233.73} & \textbf{1643.47} & - & -\\ \hline
A2C (ReLU)          & 1132.85 & 1341.84 & 44.81 & 11.80 & -15.77\\ 
A2C (ELU)           & 1280.25 & 1489.04 & -1086.8 & 27.14 & -15.73\\
A2C (Sigmoid)       & 832.32 & 1494.26 & 155.51     & -3.46 & -13.25\\
A2C (Tanh)          & 622.83 & 1495.83 & 140.67     & 25.12 & -17.84\\
A2C+Pop-Art     & -455.57 & 207.99       & -572.62 & -16.27 & -148.40\\ 
A2C+ANS    & \textbf{3689.64} & \textbf{1533.48} & \textbf{1760.07} & \textbf{115.12} & \textbf{1000.57}\\ 
\end{tabular}
\end{adjustbox}
\caption{Results of different frameworks applied to RL models. Final average across 5 trials of returns across the last 100 trajectories after $N$ frames.}
\label{table:cum_reward}
\end{table*}

\begin{figure*}[!tbh]
\centering
\begin{subfigure}{0.33\textwidth}
    \centering
    \includegraphics[width=1\linewidth]{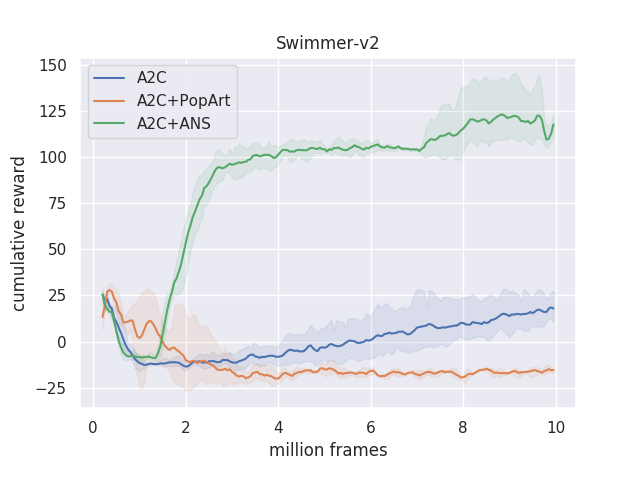}
    \caption{Swimmer-v2 (A2C)}
\end{subfigure}%
\begin{subfigure}{0.33\textwidth}
    \centering
    \includegraphics[width=1\linewidth]{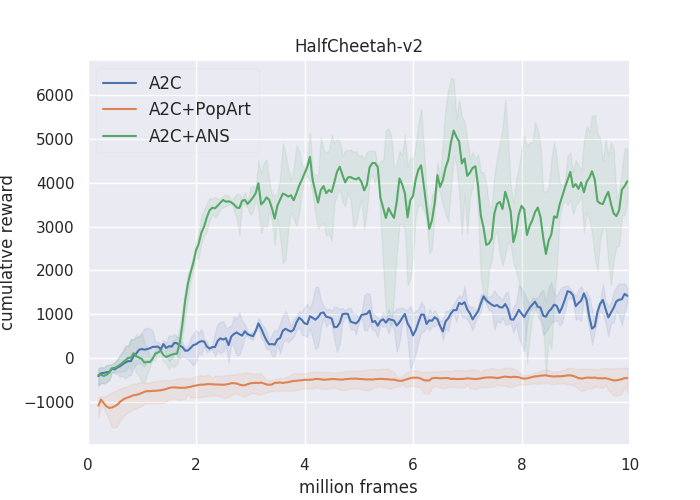}
    \caption{HalfCheetah-v2 (A2C)}
\end{subfigure}%
\begin{subfigure}{0.32\textwidth}
    \centering
    \vspace{2mm}
    \includegraphics[width=1\linewidth]{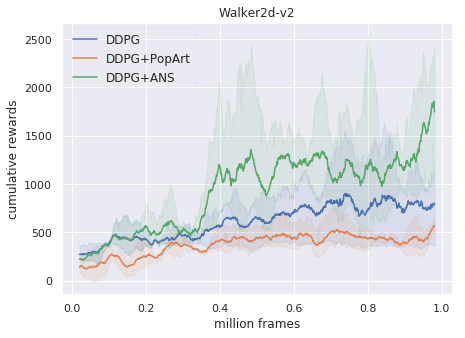}
    \caption{Walker2d-v2 (DDPG)}
\end{subfigure}
\caption{Performance comparison between original models, models with Pop-Art, and models with ANS.}
\label{fig:ans}
\end{figure*}

As indicated in Figure \ref{fig:ans}, our experiments show that the performance of ReLU is competitive with other activations although it is faster to compute. Besides, RL models equipped with Pop-Art even perform worse than original models. It indicates that unit variance and zero mean may not be the appropriate setting for ReLU networks. The other reason could be that Pop-Art is at first crafted for value-based RL approaches such as Dueling DQN \cite{wang2015dueling} instead of actor-critic models. On the other hand, ANS reaches competitive performance with the proper reward scales in ReLU experiments. The most important thing is that we have no prior information about how the reward scale should be and ANS is able to find it with limited number of frames. ANS also outperforms models with ELU, sigmoid and tanh when no reward scaling is applied. The advantages of ReLU networks with ANS over them are threefold: (1) It is easier to compute and therefore the learning process is faster. (2) It is able to preserve the sparsity property of ReLU. (3) The performance is better by a large margin. Another observation is that in most of the cases, ANS starts to shrink the reward scale after continuously increasing till reaches 64, which agrees with our previous experiments that the optimal scale is less than 64.

\section{Conclusion}

\noindent 
Here is the brief summary of the contribution of this work
\begin{itemize}
    \item We conduct thorough experiments and provide analysis to show that ReLU is a competitive activation function to used for training DRL, but its performance can be influenced by the scale of the rewards. 
    \item The idea of pseudo-dying ReLU is proposed for DRL models to evaluate the dying-ReLU situation, and empirically we have shown that it is correlated with the scale of the rewards as well as the final cumulative rewards.
    \item We propose the adaptive network scaling framework, which can be incorporated with actor-critic approaches with ReLU networks to efficiently find a suitable reward scale of better performance. Empirically, we show that ANS improves the performance with respect to the original models by a large margin.
\end{itemize}
Future works include the extension to analyze the effects of reward scaling to saturating activations such as \emph{sigmoid} and \emph{tanh}, and design a more general strategy to find suitable reward scale for different activation functions. 



\newpage
\balance
\bibliography{reference}
\bibliographystyle{aaai}
\end{document}